\pdfoutput=1 
\documentclass[11pt]{article}

\usepackage[font=times, citestyle=numeric]{kurbanlab}

\DeclareAffiliation{hbku}{%
  College of Science and Engineering, Hamad Bin Khalifa University, Doha, Qatar}

\DeclareAffiliation{tamu}{%
  Department of Computer and Electrical Engineering,
  Texas A\&M University, College Station, TX, USA}

\DeclareAffiliation{iub}{%
  Luddy School of Informatics, Computing, and Engineering,
  Indiana University Bloomington, Bloomington, IN, USA}

\title{Rank Reversal in Multilingual LLM Judges}
\Subtitle{A Label-Free Double-Centering Calibrator}
\RunningTitle{Rank reversal in multilingual LLM judges}

\Author[equal, orcid=0009-0002-6215-2483]{Alhasan Mahmood}{hbku}
\Author[equal, orcid=0009-0009-2183-6193]{Samir Abdaljalil}{tamu}
\Author[corresponding=hkurban@hbku.edu.qa, orcid=0000-0003-3142-2866]{Hasan Kurban}{hbku}

\CodeURL{https://github.com/KurbanIntelligenceLab/multilingual-judge-calibration}
\Venue{Preprint}

\begin{document}
\maketitle

\begin{abstract}
Multilingual LLM judges produce different evaluator-backbone rankings depending on the prompt language: on an eight-language Agent-as-a-Judge benchmark, the top-ranked backbone alternates across English, Arabic, Chinese, Hindi, Japanese, Spanish, Turkish, and Swahili, and 7 of 15 backbone pairs show statistically significant pairwise rank reversal. We treat this as a measurement problem. The multilingual judge score decomposes additively into task difficulty, backbone skill, and a language-backbone interaction term, the last of which is recoverable without human labels by double-centering the cell-mean score matrix. We make this estimator (\textbf{Consensus-Based Calibration}, CBC) explicit, give an $O(1/\sqrt{n})$ finite-sample concentration bound with variance constant $(1-\tfrac{1}{m})(1-\tfrac{1}{k})$, and show that it is unbiased even when task-language interactions are present. Across 7{,}920 judge runs (6 backbones, 8 languages, 55 tasks, 3 frameworks), CBC raises held-out cross-task rank consistency $\tau$ from 0.650 to 0.902 and agrees with the held-out additive-model oracle in 100\% of per-language decisions versus 68.5\% raw; these are consistency diagnostics, not human-grounded correctness measures. On a separately collected M-RewardBench panel (7 languages, 1{,}500 items per language, 10{,}500 language-item instances, 5 evaluators), panel agreement with the public human gold preferences rises from 68.7\% to 76.6\% (gain 7.9 percentage points, 95\% CI $[6.0, 9.9]$), our strongest external evidence of downstream usefulness. The estimator is the standard two-way ANOVA interaction-recovery operation under sum-to-zero contrasts; our contribution is its application as a label-free post-hoc calibrator for multilingual LLM judges, an explicit finite-sample concentration bound, and an unbiasedness result that holds even under task-language misspecification. Code is available at \url{https://github.com/KurbanIntelligenceLab/multilingual-judge-calibration}
\end{abstract}

\printkeywords

\section{Introduction}
The use of LLM-based judges to evaluate model and agent outputs is now common across recent benchmarks \citep{zhuge2024agentjudge,zheng2024judging,tan2025judgebench,lambert2024rewardbench}. As this deployment extends across languages, multiple studies have shown that judge behavior is not stable across the prompt language: \citet{hada-etal-2024-large} report systematic score inflation in multilingual evaluators against native-speaker references; \citet{fu2025reliable} document weak cross-lingual consistency (Fleiss' $\kappa \approx 0.3$) across 25 languages; \citet{singh-etal-2025-global} show that multilingual benchmark construction itself encodes linguistic and cultural bias; and \citet{mahmood2026multilingualpromptlocalizationagentasajudge} report concrete backbone-ranking reversals when judge prompts are localized.

What is less settled is what to do about it without expensive new annotations. \citet{hada-etal-2024-large} calibrate against 20{,}000 human judgments; \citet{fu2025reliable} propose an ensemble; \citet{xu2026judgeaware} extend Bradley--Terry--Luce with judge-specific discrimination for pairwise comparisons; \citet{li2025calibraeval} address position-induced selection bias. None directly targets the additive language-backbone \emph{interaction} that appears when several backbones are scored on the same items across languages, even though that is the structure produced by any multi-evaluator multilingual benchmark.

This paper studies that structure. We adopt the standard two-way additive layout from analysis of variance \citep{scheffe1959analysis,searle1971linear}, $S(t, \ell, b) = \mu(t) + \alpha(b) + \beta(\ell, b) + \gamma(t, \ell) + \epsilon$, in which $\beta(\ell, b)$ is the language-backbone interaction we target when a language-invariant evaluator ranking is the operational goal. Under sum-to-zero normalization, $\beta(\ell, b)$ is identified by double-centering the cell-mean matrix. The estimator is standard in two-way ANOVA; the contribution is its application as a label-free post-hoc calibrator for multilingual LLM-judge matrices, together with a finite-sample concentration bound and a misspecification-robustness analysis. We call this Consensus-Based Calibration (CBC).

This two-way model does not cover a genuine task--language--backbone interaction $\gamma(t,\ell,b)$: backbone-specific task adaptation can be absorbed into the estimated language-backbone interaction rather than separated by double-centering. We treat this three-way effect as a central misspecification risk, not as evidence that every observed interaction is evaluative bias.

The internal experiments build on the previously introduced five-language multilingual Agent-as-a-Judge setting \citep{mahmood2026multilingualpromptlocalizationagentasajudge}. That prior work supplies the benchmark infrastructure and original language panel; this paper adds Japanese, Spanish, and Swahili, conducts the rank-reversal analysis, formalizes and evaluates CBC, and adds the separately collected M-RewardBench validation panel.

Our contributions are:

\begin{enumerate}
\item \textbf{Empirical characterization.} On the expanded eight-language Agent-as-a-Judge benchmark (7{,}920 judge runs across 6 backbones, 55 tasks, and 3 frameworks), 7 of 15 backbone pairs exhibit pairwise rank reversal (Benjamini--Hochberg corrected at FDR $=0.05$); the empirical top-ranked backbone alternates across all 8 languages.

\item \textbf{A label-free estimator with explicit guarantees.} CBC is one operation on the multi-evaluator score matrix and requires no human labels. Under independent homoskedastic Gaussian noise, $|\hat\beta(\ell, b) - \beta(\ell, b)|$ concentrates at $O(1/\sqrt{n})$ with explicit variance constant $(1{-}1/m)(1{-}1/k)$ (Prop.~\ref{prop:convergence}). The central non-trivial guarantee, beyond classical ANOVA, is that the estimator remains \emph{unbiased} even when task-language interactions $\gamma(t,\ell)$ are present (Prop.~\ref{prop:misspec}).

\item \textbf{Decision-level and external validation.} CBC raises held-out cross-task rank consistency $\tau$ from $0.650$ to $0.902$ and agrees with the held-out additive-model oracle in 100\% of per-language decisions versus 68.5\% raw; this is model-based consistency, not human-grounded correctness. On a separately collected M-RewardBench panel of 1{,}500 items per language (10{,}500 language-item instances across 7 languages and 5 evaluator backbones), CBC raises $\tau$ from $0.430$ to $0.900$, while external human-anchor agreement with the public gold preferences \citep{gureja-etal-2025-rewardbench} rises from 68.7\% to 76.6\% (+7.9 percentage points, 95\% CI $[6.0, 9.9]$), our strongest external evidence of downstream usefulness.
\end{enumerate}

\noindent
\textbf{Scope.} CBC removes the language-backbone interaction term and nothing else. This correction is appropriate when the deployment objective is a language-invariant evaluator ranking; if language-specific evaluator specialization is itself the target, removing the interaction can remove meaningful capability differences, so CBC should be treated as a sensitivity analysis. A shared language-level shift $g(\ell)$ where all evaluators systematically over- or under-score one language is invisible to double-centering; correcting it would require an external anchor.
\section{Related Work}
\paragraph{Multilingual LLM-judge evaluation.}
\citet{hada-etal-2024-large} calibrate multilingual evaluators against 20{,}000 native-speaker judgments and report systematic score inflation in uncalibrated evaluators. \citet{fu2025reliable} survey 25 languages, document weak cross-lingual consistency (Fleiss' $\kappa \approx 0.3$), and propose an ensemble strategy. \citet{sheth2026crosslingual} introduce a Universal Criteria Set for cross-lingual transfer with minimal supervision. Multilingual meta-evaluation benchmarks \citep{doddapaneni2025cross,son2024mmeval} and analyses of multilingual benchmark construction bias \citep{singh-etal-2025-global} provide additional context. Known judge biases (position, verbosity, self-enhancement) are characterized in \citet{zheng2024judging}; benchmarks for judge quality include RewardBench \citep{lambert2024rewardbench}, JudgeBench \citep{tan2025judgebench}, and MT-Bench-101 \citep{bai-etal-2024-mt}. Our work is label-free, post-hoc on an already-collected multi-backbone score matrix, and targets the language-backbone interaction term that arises in any multi-evaluator multilingual benchmark.

\paragraph{Label-free judge calibration and aggregation.}
\citet{xu2026judgeaware} extend the Bradley--Terry--Luce model with judge-specific discrimination for pairwise comparisons. \citet{li2025calibraeval} address position-induced selection bias in pairwise LLM judges. The Dawid--Skene tradition \citep{dawid1979maximum,whitehill2009glad,hovy2013mace,paun2018comparing} and learning-from-crowds methods \citep{raykar2010learning,li-2019-truth} estimate annotator reliability on categorical labels without ground truth; Item Response Theory \citep{embretson2013item,pmlr-v235-maia-polo24a} models rater-item interactions for psychometric measurement. Empirical work shows substantial cross-task LLM-judge variability \citep{bavaresco2024llms}, and disagreement-aware NLP evaluation \citep{pavlick2019inherent,leonardelli-etal-2023-lewidi} emphasizes preserving annotator variation. Our setting is continuous pointwise score matrices with a structured language-backbone interaction, where the relevant object is an interaction term in a two-way layout rather than a pairwise reliability parameter.

\paragraph{Two-way layouts and agentic code evaluation.}
The additive decomposition is the standard two-way ANOVA layout \citep{scheffe1959analysis,searle1971linear}; double centering under sum-to-zero contrasts is the standard interaction-recovery operation. Classical references target F-test inference under a correctly specified model. The two formal properties we use, a finite-sample uniform bound on $|\hat\beta-\beta|$ across all $mk$ cells with explicit variance constant $(1{-}\tfrac{1}{m})(1{-}\tfrac{1}{k})$ (Prop.~\ref{prop:convergence}) and unbiasedness under an unmodeled task-language interaction $\gamma(t,\ell)$ (Prop.~\ref{prop:misspec}), are not standard there. We build on Agent-as-a-Judge \citep{zhuge2024agentjudge} and a previously introduced multilingual Agent-as-a-Judge setting \citep{mahmood2026multilingualpromptlocalizationagentasajudge}; M-RewardBench \citep{gureja-etal-2025-rewardbench} supplies the public preference instances for our external validation panel.

\section{Score Model and Identifiability}
\label{sec:framework}

\subsection{Setup and Notation}

\paragraph{Terminology.} We use \emph{evaluator backbone} $b$ for the LLM that scores an item, \emph{judge framework} $f$ for the evaluation harness and prompt/rubric protocol (MetaGPT, GPT-Pilot, or OpenHands), and \emph{evaluated agent output} for the code artifact produced for task $t$ that the judge scores. For the internal benchmark, the CBC task-level matrix averages the three framework-level scores for each $(t,\ell,b)$; framework effects are retained as fixed effects in framework-level variance checks. Thus our rankings compare evaluator backbones on shared evaluated outputs, not judge frameworks or the underlying agents.

We have $n$ tasks $\mathcal{T}$, $m$ backbones $\mathcal{B}$, and $k$ languages $\mathcal{L}$. The judge produces a score $S(t, \ell, b) \in [0, 100]$, modeled as a two-way additive layout with interaction \citep{scheffe1959analysis,searle1971linear}:
\begin{equation}
S(t, \ell, b) = \mu(t) + \alpha(b) + \beta(\ell, b) + \gamma(t, \ell) + \epsilon,
\label{eq:full}
\end{equation}
where $\mu(t)$ is task difficulty, $\alpha(b)$ is backbone skill, $\beta(\ell, b)$ is the \emph{language-backbone interaction} targeted by CBC, $\gamma(t, \ell)$ is a residual task-language effect (any language-uniform offset across backbones is absorbed into $g(\ell) = \frac{1}{n}\sum_t \gamma(t,\ell)$), and $\epsilon$ is mean-zero noise. A framework term $\phi(f)$ for pooling across judge frameworks does not depend on both $\ell$ and $b$ and cancels under the estimator below. Eq.~\ref{eq:full} is the standard Type~III two-way ANOVA layout with factors backbone and language and an interaction term.

The interpretation of $\beta(\ell,b)$ as an evaluative bias is operational rather than intrinsic: it is appropriate when deployment requires a shared evaluator ranking for the same items across languages. If a language-backbone interaction instead reflects genuine language-specific evaluator capability or task competence, it should not automatically be removed; in that setting, CBC is best treated as a sensitivity analysis.

\subsection{Pairwise Rank Reversal}

The interaction term $\beta(\ell, b)$ is empirically non-trivial whenever per-language backbone rankings disagree. Writing $d_{ij}(\ell) \triangleq \mathbb{E}_t[S(t, \ell, b_i) - S(t, \ell, b_j)]$ for the per-language gap between two backbones, measured in score points:

\begin{definition}[Pairwise rank reversal]
\label{def:diversity}
Backbones $b_i, b_j$ exhibit \emph{pairwise rank reversal} with respect to a language set $\mathcal{L}$ if there exist $\ell_a, \ell_b \in \mathcal{L}$ such that
$d_{ij}(\ell_a) \cdot d_{ij}(\ell_b) \leq -\delta$ for some strength $\delta > 0$.
\end{definition}

Because $\delta$ is a product of two score gaps, its units are squared score points ($\mathrm{points}^2$). A single satisfying pair rules out universal dominance between those two backbones. On the expanded eight-language Agent-as-a-Judge benchmark (Section~\ref{sec:experiments}, Table~\ref{tab:diversity}), 7 of 15 backbone pairs satisfy this condition under task-level one-sided tests with Benjamini--Hochberg correction at FDR $= 0.05$; the empirical top-ranked backbone alternates across the eight languages, so no universal winner exists in the observed data. The strongest reversal is GPT-4o vs.\ GPT-5.4: GPT-4o leads in English by $+35.63$ points and trails in Spanish by $-11.25$ points, giving $\delta = 400.79~\mathrm{points}^2$ from the unrounded means. Figure~\ref{fig:interaction_heatmap} visualizes the corresponding centered language-backbone interaction matrix $\hat\beta(\ell, b)$ recovered by CBC; these $\hat\beta$ cells are not the raw pairwise gaps $d_{ij}$.

\begin{figure}[t]
\centering
\includegraphics[width=0.98\columnwidth]{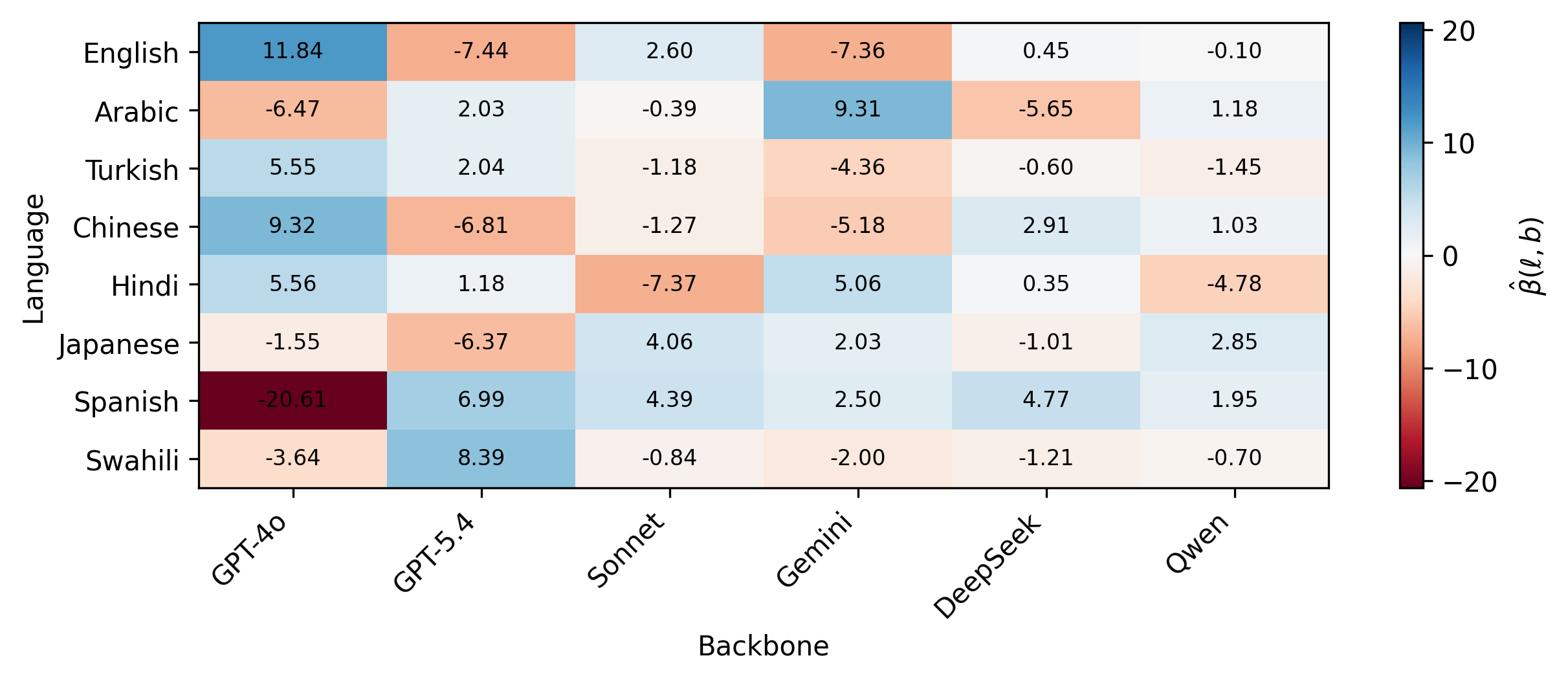}
\caption{Estimated centered language-backbone interaction $\hat\beta(\ell, b)$ from the expanded multilingual Agent-as-a-Judge benchmark (8 languages, 6 backbones), computed by double-centering the framework-averaged task scores. Positive values: backbone scores higher than expected in that language. These cells are distinct from the raw pairwise gaps $d_{ij}$ reported in score points. Strongest cells: GPT-4o in Spanish ($-20.61$) vs.\ English ($+11.84$), Gemini in Arabic ($+9.31$), GPT-5.4 in Swahili ($+8.39$) vs.\ English ($-7.44$). The simultaneous radius is about $10.0$ points at $\varepsilon=0.05$ under Proposition~\ref{prop:convergence}, so moderate cells should be interpreted cautiously (Section~\ref{sec:limitations}).}
\label{fig:interaction_heatmap}
\end{figure}

\subsection{Identifiability Without Human Labels}

Assume a complete balanced panel: every task $t \in \mathcal{T}$ is scored by every language-backbone pair $(\ell, b) \in \mathcal{L} \times \mathcal{B}$. Under the simplified model $\gamma(t, \ell) \approx 0$:
\begin{equation}
S(t, \ell, b) = \mu(t) + \alpha(b) + \beta(\ell, b) + \epsilon.
\label{eq:simplified}
\end{equation}

\begin{proposition}[Identifiability of the interaction matrix]
\label{thm:identifiability}
The standard two-way ANOVA interaction-identifiability result \citep{scheffe1959analysis,searle1971linear} specializes to our setting as follows. Assume $m \geq 2$, $k \geq 2$, a complete balanced panel, and $\mathbb{E}[\epsilon(t,\ell,b)] = 0$. Under the sum-to-zero normalization
$\sum_{\ell} \beta(\ell,b) = 0 \;\forall b$ and $\sum_b \beta(\ell,b) = 0 \;\forall \ell$,
the interaction matrix $\beta(\ell,b)$ is uniquely identifiable from the population cell means $M(\ell,b) \triangleq \frac{1}{n}\sum_t \mathbb{E}[S(t,\ell,b)]$ via double centering. The proof is in Appendix~\ref{app:proofs}.
\end{proposition}

\paragraph{Scope of identifiability.} Under the sum-to-zero constraints, $\beta(\ell,b)$ is uniquely identified from the population cell means for \emph{observed} language-backbone cells; the absolute levels of $\mu(t)$ and $\alpha(b)$ are not separately identified, but this indeterminacy does not affect $\beta$. A three-way effect $\gamma(t,\ell,b)$ is outside the two-way model and cannot be separated from $\beta$ by CBC.

\section{Consensus-Based Calibration}
\label{sec:cbc}
CBC removes the language-backbone interaction term $\beta(\ell,b)$ when the operational goal is a shared evaluator ranking across languages. It does \emph{not} detect or remove a shift shared across all backbones in a language. For example, if every evaluator backbone systematically underscored Swahili by 5 points, that shared language-level shift would cancel under double centering and CBC would leave it unchanged. Correcting such a shared shift requires an external anchor, such as human labels or a trusted reference evaluator, which is outside the scope of the present label-free interaction-calibration setting.

\subsection{Algorithm}

Proposition~\ref{thm:identifiability} yields a constructive estimator:

\textbf{Step 1.} Compute $\bar{S}(\ell, b) = \frac{1}{n}\sum_t S(t, \ell, b)$.

\textbf{Step 2.} Estimate the interaction via double centering:
\begin{equation}
\hat\beta(\ell, b) = \bar{S}(\ell, b) - \bar{S}(\cdot, b) - \bar{S}(\ell, \cdot) + \bar{S}(\cdot, \cdot).
\label{eq:residual}
\end{equation}

\textbf{Step 3.} Calibrate: $\hat{S}(t, \ell, b) = S(t, \ell, b) - \hat\beta(\ell, b)$.

CBC requires \textit{zero human annotations}. It uses only multi-backbone evaluation data that labs already collect.

\subsection{Convergence and Consistency}

\begin{proposition}[Convergence of CBC]
\label{prop:convergence}
Assume Eq.~\ref{eq:simplified}, a complete balanced panel with $n$ tasks per language-backbone cell, and independent homoskedastic Gaussian noise $\epsilon(t,\ell,b) \sim \mathcal{N}(0,\sigma^2)$. Then for each fixed $(\ell,b)$ and every $x > 0$,
{\footnotesize
\begin{equation}
\Pr\!\left(|\hat\beta(\ell,b){-}\beta(\ell,b)| \geq x\right)
\leq
2\exp\!\left(
\frac{-n x^2}{2\sigma^2(1{-}\tfrac{1}{m})(1{-}\tfrac{1}{k})}
\right).
\end{equation}
}
Consequently, with probability at least $1-\varepsilon$,
{\footnotesize
\begin{equation}
\max_{\ell,b} |\hat\beta(\ell, b) - \beta(\ell, b)|
\leq
\sigma\sqrt{\frac{2(1{-}\tfrac{1}{m})(1{-}\tfrac{1}{k})\log(2mk/\varepsilon)}{n}}.
\end{equation}
}
\end{proposition}

The exact constant comes from the variance of the double-centering operator: $\mathrm{Var}[\hat\beta(\ell,b)-\beta(\ell,b)] = \frac{\sigma^2}{n}(1-\frac{1}{m})(1-\frac{1}{k})$. For the expanded panel used in our main experiments ($m{=}6$, $k{=}8$, $n{=}55$), the simultaneous high-probability radius is about $10.0$ points at $\varepsilon{=}0.05$ with $\hat\sigma=22.24$ estimated from the benchmark score dispersion (Appendix~\ref{app:ablations}). This model-dependent numerical radius is conditional on independent homoskedastic Gaussian noise; it is not a distribution-free guarantee and need not transfer to heteroskedastic or heavy-tailed judge noise. The weaker consistency result below does not require Gaussian tails.

\begin{proposition}[Consistency]
\label{prop:consistency}
Without the Gaussian assumption, if Eq.~\ref{eq:simplified} holds with independent tasks and $\mathbb{E}[|\epsilon(t,\ell,b)|] < \infty$, then $\hat\beta(\ell, b) \xrightarrow{p} \beta(\ell, b)$ as $n \to \infty$ for every $(\ell, b)$ by continuous mapping on the four sample marginal means (Appendix~\ref{app:proofs}). The $O(1/\sqrt{n})$ rate is the standard parametric rate for this sample-mean construction; we do not claim a minimax lower bound here.
\end{proposition}

\subsection{Model Misspecification}

\begin{proposition}[Robustness to Task-Language Effects]
\label{prop:misspec}
Under the full model (Eq.~\ref{eq:full}), the CBC estimator $\hat\beta(\ell, b)$ is \textit{unbiased} for $\beta(\ell, b)$: the task-language interaction $\gamma(t, \ell)$ cancels exactly in the double-centering operation (Eq.~\ref{eq:residual}), because $\gamma$ does not depend on $b$.
\end{proposition}

\begin{proof}
Define $g(\ell) = \frac{1}{n}\sum_t \gamma(t, \ell)$. Under Eq.~\ref{eq:full}, $\bar{S}(\ell, b) = \bar\mu + \alpha(b) + \beta(\ell, b) + g(\ell) + \bar\epsilon$. In the double-centering operator of Eq.~\ref{eq:residual}, $g(\ell)$ appears in $\bar{S}(\ell, b)$ and $\bar{S}(\ell, \cdot)$ with coefficients $+1$ and $-1$, and $\bar{g}$ appears in $\bar{S}(\cdot, b)$ and $\bar{S}(\cdot, \cdot)$ with coefficients $-1$ and $+1$. All $\gamma$-derived terms cancel, leaving $\hat\beta(\ell, b) = \beta(\ell, b)$ plus mean-zero noise.
\end{proof}

\begin{remark}
This robustness result is what separates CBC's formal contribution from a direct application of two-way ANOVA, whose standard reference inference assumes the model is correctly specified. CBC handles task-language interactions without bias in the point estimate, though the simplified-model error bars of Proposition~\ref{prop:convergence} need not carry over unchanged under the misspecified model. Proposition~\ref{prop:misspec} does not cover a genuine \emph{three-way} interaction $\gamma(t, \ell, b)$, which can arise from backbone-specific task adaptation and lies outside the two-way additive model. Under such an effect, double centering may absorb real performance heterogeneity as well as evaluation interaction, so additional structure would be needed to separate the two.
\end{remark}

\section{Experiments}
\label{sec:experiments}
\subsection{Data}

We build on the previously introduced five-language multilingual Agent-as-a-Judge setting \citep{mahmood2026multilingualpromptlocalizationagentasajudge}. The present paper adds Japanese, Spanish, and Swahili, yielding an expanded internal benchmark with 7{,}920 runs across 6 backbones, 8 languages, 3 judge frameworks (MetaGPT, GPT-Pilot, and OpenHands), and 55 DevAI tasks from one agentic software-engineering benchmark family. The six internal evaluator identifiers are \texttt{gpt-4o}, \texttt{gpt-5.4}, \texttt{claude-sonnet-4.6}, \texttt{gemini-3-flash-preview}, \texttt{deepseek-v3.2}, and \texttt{qwen3.5-9b}, displayed in the paper as GPT-4o, GPT-5.4, Sonnet, Gemini, DeepSeek, and Qwen. Because provider billing exports were not preserved consistently, we report estimated list-price API cost rather than billed totals: using recorded token counts from the saved run artifacts, the added three-language extension is estimated at \$54.93 (from 25.68M input plus 6.64M output tokens times contemporaneous public list prices; Appendix~\ref{app:costs}).

\paragraph{Statistical methodology.} Because CBC estimates a language-backbone interaction term $\beta(\ell, b)$, our main evaluation keeps all eight languages in the observed set and uses task-level bootstrap resampling: in each of 1{,}000 replicates, we sample tasks with replacement to form the training set, estimate $\hat\beta$ from all eight observed languages on those sampled tasks, and evaluate calibrated rankings on the out-of-bag tasks from the same language set. Our primary metric is mean pairwise Kendall rank correlation $\tau$ across all 28 language pairs, computed from the backbone rankings on the held-out tasks. We additionally report an exploratory leave-one-language-out (LOLO) extrapolation diagnostic: for each held-out language $\ell_h$, we estimate $\hat\beta$ on the remaining seven languages using the sampled training tasks, calibrate those seven languages, and compare the held-out ranking on out-of-bag tasks against each retained language via mean Kendall $\tau$. For the held-out language we consider two simple heuristics, zero-shot $\hat\beta(\ell_h,b)=0$ and backbone-wise mean imputation. Under CBC's sum-to-zero normalization, however, the backbone-wise mean of $\hat\beta(\cdot,b)$ over the retained languages is exactly zero, so the two LOLO heuristics coincide up to numerical noise.

\subsection{Verifying Theoretical Conditions}

\paragraph{Rank reversal.}
Using framework-averaged task scores, 7 of 15 backbone pairs exhibit pairwise rank reversal ($\delta > 0$ in Definition~\ref{def:diversity}); see Table~\ref{tab:diversity}. The raw witness gaps are in score points: for GPT-4o vs.\ GPT-5.4, English favors GPT-4o ($d = +35.63$) and Spanish favors GPT-5.4 ($d = -11.25$), giving $\delta = 400.79~\mathrm{points}^2$ from the unrounded means. We use the intersection-union test over the two witness languages and apply Benjamini--Hochberg correction at FDR $=0.05$ to the 15 pairwise tests; all 7 observed reversal pairs remain significant after correction. The centered interaction cells $\hat\beta$ shown in Figure~\ref{fig:interaction_heatmap} are a different quantity from these raw pairwise gaps. The empirical top-ranked backbone alternates between GPT-4o (English, Chinese, Turkish) and Gemini (Arabic, Hindi, Japanese, Spanish, Swahili), so no universal winner exists in the observed data.

\begin{table}[t]
\centering\footnotesize
\renewcommand{\arraystretch}{1.15}
\setlength{\tabcolsep}{4pt}
\caption{Rank-reversal strength ($\delta$ in $\mathrm{points}^2$) for backbone pairs on the expanded eight-language benchmark. Each test uses one-sided task-level tests in the two witness languages (the language pair attaining the most negative product $d_{ij}(\ell_a)\cdot d_{ij}(\ell_b)$), combined by the intersection-union rule; adj.\ $p$ uses Benjamini--Hochberg over all 15 pairs. Witness language pairs are listed in Appendix~\ref{app:proofs}.}
\label{tab:diversity}
\begin{tabular}{@{}lrr@{}}
\toprule
\textbf{Pair} & \textbf{$\delta$ ($\mathrm{points}^2$)} & \textbf{Adj. $p$} \\
\midrule
GPT-4o vs.\ GPT-5.4    & 400.8 & $9.7{\times}10^{-8}$ \\
GPT-4o vs.\ Sonnet     & 221.1 & 0.0019 \\
GPT-4o vs.\ Gemini     & 312.0 & 0.0019 \\
GPT-4o vs.\ DeepSeek   & 140.7 & 0.0031 \\
GPT-4o vs.\ Qwen       & 0.0   & 1.000 \\
GPT-5.4 vs.\ Sonnet    & 95.9  & $7.5{\times}10^{-6}$ \\
GPT-5.4 vs.\ Gemini    & 0.0   & 1.000 \\
GPT-5.4 vs.\ DeepSeek  & 71.8  & 0.0119 \\
GPT-5.4 vs.\ Qwen      & 0.0   & 1.000 \\
Sonnet vs.\ Gemini     & 0.0   & 1.000 \\
Sonnet vs.\ DeepSeek   & 38.1  & 0.0246 \\
Sonnet vs.\ Qwen       & 0.0   & 1.000 \\
Gemini vs.\ DeepSeek   & 0.0   & 1.000 \\
Gemini vs.\ Qwen       & 0.0   & 1.000 \\
DeepSeek vs.\ Qwen     & 0.0   & 1.000 \\
\bottomrule
\end{tabular}
\end{table}

\paragraph{Simplified model fit.}
We fit Eq.~\ref{eq:simplified} and Eq.~\ref{eq:full} by OLS on all 7{,}920 framework-level observations with framework fixed effects. The simplified model achieves $R^2 = 0.691$; the full model with task-language interactions reaches $R^2 = 0.705$. The simplified two-way structure captures a substantial majority of the variance; task-language effects contribute only $\Delta R^2 \approx 0.015$.

\subsection{Calibration Results}

Our primary evaluation estimates out-of-bag cross-task stability within the observed eight-language panel. Using the full observed language set and splitting over tasks, raw scores achieve mean held-out cross-task rank consistency $\tau = 0.650$, while CBC reaches $\tau = 0.902$. After full-fit calibration, all eight languages align on the same backbone order (Figure~\ref{fig:cbc_before_after}); this $\tau=1.000$ ordering is an in-sample sanity check, not an independent generalization estimate or a human-grounded correctness measure. Separately, the exploratory LOLO diagnostic rises from $\tau = 0.649$ under Raw to $\tau = 0.740$ under the LOLO zero-shot CBC heuristic (best: Arabic $0.750 \to 0.898$; hardest: Spanish $0.411 \to 0.504$). Mean imputation coincides with zero-shot here under CBC's sum-to-zero normalization, so this diagnostic does not establish zero-shot transfer to unseen languages. We compare against post-hoc baselines that match score adjustment on a continuous matrix: quantile normalization \citep{bolstad2003comparison} and ComBat-EB batch correction \citep{johnson2007combat}, alongside simpler controls (Table~\ref{tab:calibration}). Per-language normalization, z-score, quantile normalization, ensemble, backbone-only normalization, and random control are diagnostic controls for generic location/scale alignment or aggregation; ComBat-EB is the strongest substantive continuous post-hoc comparator compatible with this matrix. The oracle is an unattainable upper bound using the held-out-task interaction.

\begin{table}[t]
\centering\footnotesize
\renewcommand{\arraystretch}{1.15}
\setlength{\tabcolsep}{3pt}
\caption{Calibration on the observed eight-language benchmark (higher is better). In each bootstrap replicate, methods are fit on sampled training tasks and evaluated on out-of-bag held-out tasks. ``Full-fit$^\ast$'' is an in-sample diagnostic on all 55 tasks. The oracle subtracts $\beta$ on the held-out tasks themselves and is an unattainable additive-model upper bound. $^\ast$ denotes an in-sample sanity check, not an independent generalization estimate.}
\label{tab:calibration}
\begin{tabular}{@{}lccc@{}}
\toprule
\textbf{Method} & \textbf{$\tau$ $\uparrow$} & \textbf{95\% CI} & \textbf{Full-fit$^\ast$} \\
\midrule
Raw (no calib.)            & 0.650 & [0.610, 0.714] & 0.629 \\
Per-language norm.         & 0.650 & [0.610, 0.714] & 0.629 \\
ComBat-EB                  & 0.650 & [0.610, 0.714] & 0.629 \\
Z-score                    & 0.220 & [-0.038, 0.586] & -0.086 \\
Quantile norm.             & -0.097 & [-0.119, -0.068] & -0.118 \\
Ensemble                   & 0.013 & [-0.086, 0.186] & -0.081 \\
Backbone-only norm.        & 0.060 & [-0.081, 0.295] & -0.086 \\
Random control             & 0.002 & [-0.095, 0.148] & -0.067 \\
\textbf{CBC}               & \textbf{0.902} & \textbf{[0.795, 0.968]} & \textbf{1.000$^\ast$} \\
\midrule
Oracle (eval-task $\beta$) & 1.000 & --- & 1.000$^\ast$ \\
\bottomrule
\end{tabular}
\end{table}

Among the substantive continuous post-hoc comparisons, CBC significantly outperforms ComBat-EB: the paired-bootstrap difference $\tau_{\text{CBC}}-\tau_{\text{ComBat}}$ is positive in all 1{,}000 replicates ($p < 0.002$). The diagnostic controls are also informative. ComBat-EB matches Raw because, under CBC's sum-to-zero normalization, the average interaction over backbones for each language is zero, so a language-wide offset has nothing to remove. Quantile normalization is worse than Raw because aligning per-backbone score distributions across languages destroys the cross-language ranking signal CBC is designed to estimate (implementation in Appendix~\ref{app:quantile_note}); the ensemble, z-score, backbone-only, and random controls likewise do not target the interaction directly. A Dawid--Skene EM diagnostic \citep{dawid1979maximum} produced an unstable bootstrap interval (Appendix~\ref{app:ds_em_note}). Pairwise-comparison methods such as CalibraEval \citep{li2025calibraeval} and judge-aware BTL \citep{xu2026judgeaware} target a different input regime; we treat the adapted judge-aware BTL model as the substantive external pairwise comparator in Section~\ref{sec:expanded}. Supervised calibration, criterion-annotation, and position-bias methods require human labels, criterion annotations, or position-bias calibration data that are unavailable in our label-free continuous-score setting, so they are not applicable comparators here.

\begin{figure}[t]
\centering
\includegraphics[width=0.98\columnwidth]{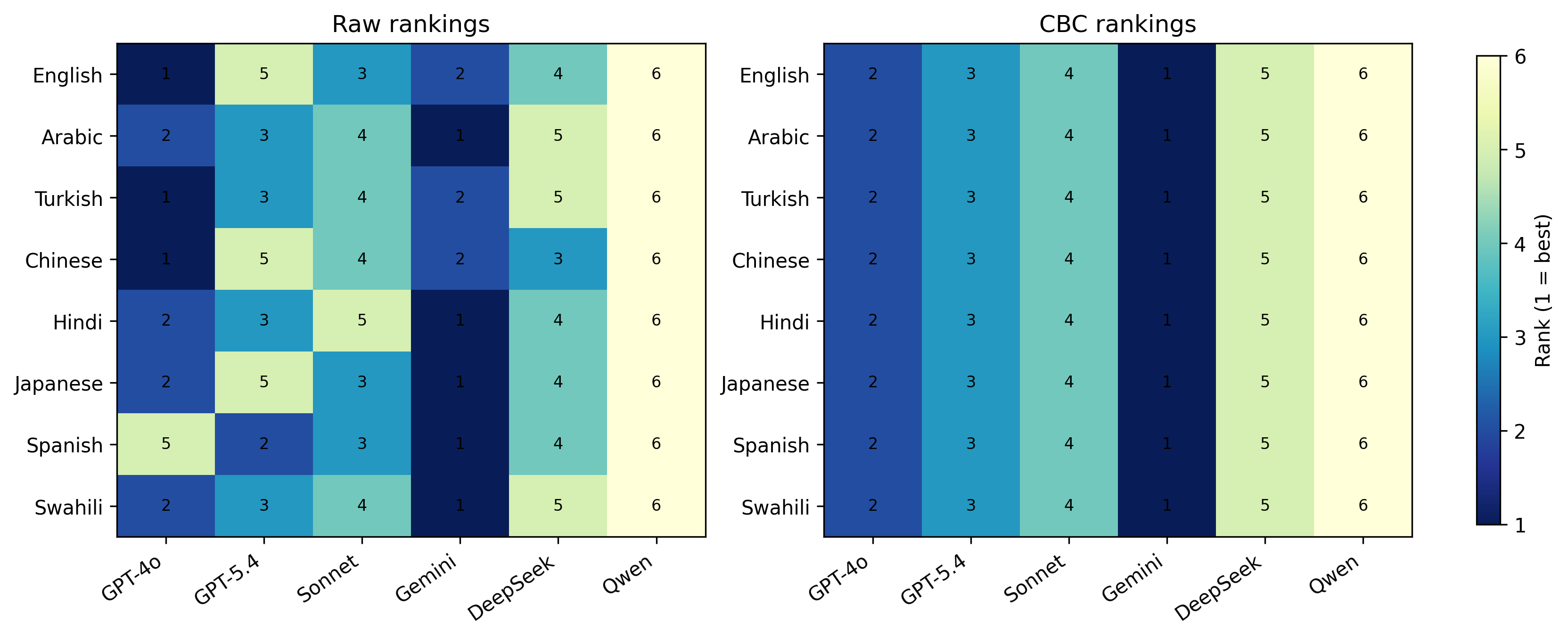}
\caption{Backbone ranks by language before (left) and after (right) CBC on the expanded eight-language benchmark. Raw rankings vary substantially across languages, whereas CBC aligns all eight languages to the same backbone order: Gemini $>$ GPT-4o $>$ GPT-5.4 $>$ Sonnet $>$ DeepSeek $>$ Qwen ($\tau=1.000$). The shared post-CBC order is a full-fit in-sample sanity check, not an independent generalization estimate.}
\label{fig:cbc_before_after}
\end{figure}

\subsection{Decision-Level Backbone Selection}

Rank consistency is useful only if it changes actual choices. To test that directly, we convert each bootstrap replicate into a per-language deployment decision. For each language, we use the training split to choose the top-ranked backbone under Raw or CBC, then evaluate that choice on the held-out tasks against the additive-model oracle winner obtained by subtracting eval-task $\beta$ on the held-out tasks themselves. This uses the same unattainable oracle as Table~\ref{tab:calibration}, but now asks a discrete question: did the method pick the backbone that the held-out calibrated scores would have preferred?

\begin{table}[t]
\centering\footnotesize
\renewcommand{\arraystretch}{1.2}
\caption{Decision-level backbone selection on the observed eight-language benchmark. Each bootstrap replicate produces 8 language-specific deployment decisions. Selection accuracy is reported with exact binomial 95\% confidence intervals over the resulting 8{,}000 language-decisions; regret is the held-out additive-model oracle-score gap between the selected backbone and the held-out oracle winner, with percentile intervals over bootstrap replicates.}
\label{tab:decision_selection}
\begin{tabular}{@{}lcc@{}}
\toprule
\textbf{Method} & \textbf{Oracle selection acc. $\uparrow$} & \textbf{Regret $\downarrow$} \\
\midrule
Raw  & 68.5\% [67.4, 69.5] & 3.23 [0.90, 5.82] \\
CBC                  & 100.0\% [99.95, 100.0] & 0.00 [0.00, 0.00] \\
\bottomrule
\end{tabular}
\end{table}

The result is sharper than the Kendall-$\tau$ view alone (Table~\ref{tab:decision_selection}). Raw rankings agree with the held-out additive-model oracle winner for only 68.5\% of language-decisions, whereas CBC agrees in all observed decisions. The raw errors are concentrated in English, Chinese, and Turkish, where Raw frequently selects GPT-4o while the held-out calibrated winner is Gemini. This is agreement with a model-based reference, not human-grounded correctness; within the stated deployment objective, it shows that CBC changes which evaluator backbone a practitioner would actually deploy.

\subsection{External Validation on M-RewardBench}
\label{sec:expanded}

\paragraph{Scaled validation panel.} As an external validation setting beyond the internal benchmark, we examined M-RewardBench \citep{gureja-etal-2025-rewardbench}, which provides aligned multilingual preference instances across 23 languages. The public release does not ship evaluator-by-language score matrices, so we built a collection pipeline that queries five provider-qualified evaluators---\path{openrouter/anthropic/claude-sonnet-4.6}, \path{deepseek/deepseek-v3.2}, \path{openrouter/google/gemini-3-flash-preview}, \path{openrouter/openai/gpt-4o-2024-08-06}, and \path{openrouter/openai/gpt-5.4}---with a fixed 1--5 pointwise rubric on chosen/rejected responses, over 7 overlapping languages (English, Arabic, Turkish, Simplified Chinese, Hindi, Japanese, Spanish) and 1{,}500 aligned items per language, or 10{,}500 language-item instances in total. We display these evaluators as Sonnet, DeepSeek, Gemini, GPT-4o, and GPT-5.4. Only two evaluator families are shared with the internal panel at the reported configuration level; because the provider routes, prompts, task sources, and language panels differ, no cross-panel consistency should be inferred. The estimated list-price API cost is \$195.3 (Appendix~\ref{app:costs}). Pipeline details and the adapted Bradley--Terry--Luce \citep{xu2026judgeaware} pairwise baseline are in Appendix~\ref{app:external_details}.

\paragraph{Calibration results.} Under the same bootstrap train/OOB protocol as the main benchmark, raw cross-language $\tau$ is moderate ($0.430$, 95\% CI $[0.371, 0.486]$); CBC raises it to $0.900$ ($[0.790, 1.000]$); the adapted judge-aware BTL pairwise baseline, our substantive comparator for the external pairwise regime, reaches only $0.405$ ($[0.352, 0.467]$), close to Raw and well below CBC (Table~\ref{tab:external_validation}). After full-fit calibration, all seven languages align on the order DeepSeek $>$ Sonnet $>$ GPT-4o $>$ GPT-5.4 $>$ Gemini with $\tau=1.000$ (Figure~\ref{fig:external_cbc_before_after}); this common full-fit order is an in-sample sanity check, not an independent generalization estimate.

\paragraph{Human-anchor check.} To add an external signal independent of our evaluator scores, we use the public human-judged chosen/rejected gold preferences in M-RewardBench itself. For each language we sample 100 items stratified by subset (54 \texttt{alpacaeval-easy}, 46 \texttt{alpacaeval-hard}), average the 5 evaluator margins per item, and ask whether the panel decision agrees with the gold preference (ties count as non-agreement). Raw agreement is 68.7\% (95\% CI $[65.4, 71.7]$); CBC-aligned panel agreement rises to 76.6\% ($[73.6, 79.3]$), a gain of 7.9 points ($[6.0, 9.9]$; Table~\ref{tab:external_human_anchor}). Because the gold labels are human-judged and released by an independent group \citep{gureja-etal-2025-rewardbench}, this human-anchor gain is our strongest external evidence that the correction is useful for a downstream target. It is supportive aggregate panel-level evidence, not a per-language diagnosis of whether an individual interaction cell reflects evaluative bias or genuine language-specific evaluator specialization; it also does not establish that every component removed by CBC is undesirable or that the calibrated scores are universally correct.

The two validation panels are materially different: the internal panel uses synthetic software-engineering tasks, three judge frameworks, and the six internal identifiers above, whereas the external panel uses public preference instances, a different pointwise rubric, seven overlapping languages, and provider-qualified evaluator routes. Together they reproduce language-conditioned rank instability; CBC improves cross-language consistency in both panels, and its human-anchor gain improves agreement with M-RewardBench gold preferences. Neither panel establishes universality across domains, languages, or evaluator families.

\begin{table}[t]
\centering\footnotesize
\renewcommand{\arraystretch}{1.15}
\setlength{\tabcolsep}{4pt}
\caption{Scaled validation on self-collected evaluator scores over public 7-language, 5-evaluator M-RewardBench instances, with the same bootstrap train/OOB protocol as the main benchmark. ``BTL$^\dagger$'' is an adapted Xu et al.\ \citeyearpar{xu2026judgeaware} pairwise-only baseline. ``Full-fit$^\ast$'' is an in-sample diagnostic on 1{,}500 items per language (10{,}500 language-item instances total). $^\ast$ denotes an in-sample sanity check, not an independent generalization estimate.}
\label{tab:external_validation}
\begin{tabular}{@{}lccc@{}}
\toprule
\textbf{Method} & \textbf{$\tau$ $\uparrow$} & \textbf{95\% CI} & \textbf{Full-fit$^\ast$} \\
\midrule
Raw (no calib.) & 0.430 & [0.371, 0.486] & 0.410 \\
Judge-aware BTL$^\dagger$ & 0.405 & [0.352, 0.467] & 0.410 \\
\textbf{CBC} & \textbf{0.900} & \textbf{[0.790, 1.000]} & \textbf{1.000$^\ast$} \\
\bottomrule
\end{tabular}
\end{table}

\begin{table}[t]
\centering\footnotesize
\renewcommand{\arraystretch}{1.15}
\setlength{\tabcolsep}{4pt}
\caption{Human-anchor validation on the M-RewardBench panel containing 1{,}500 items per language (10{,}500 language-item instances total). For each of the 7 languages we sample 100 items stratified by subset and compare the sign of the evaluator-panel mean margin against the public gold preference (higher is better). Ties count as non-agreement.}
\label{tab:external_human_anchor}
\begin{tabular}{@{}lcc@{}}
\toprule
\textbf{Method} & \textbf{Gold agree.\ $\uparrow$} & \textbf{95\% CI} \\
\midrule
Raw panel mean margin & 68.7\% & [65.4, 71.7] \\
\textbf{CBC-aligned} & \textbf{76.6\%} & \textbf{[73.6, 79.3]} \\
\bottomrule
\end{tabular}
\end{table}

\begin{figure}[t]
\centering
\includegraphics[width=0.98\columnwidth]{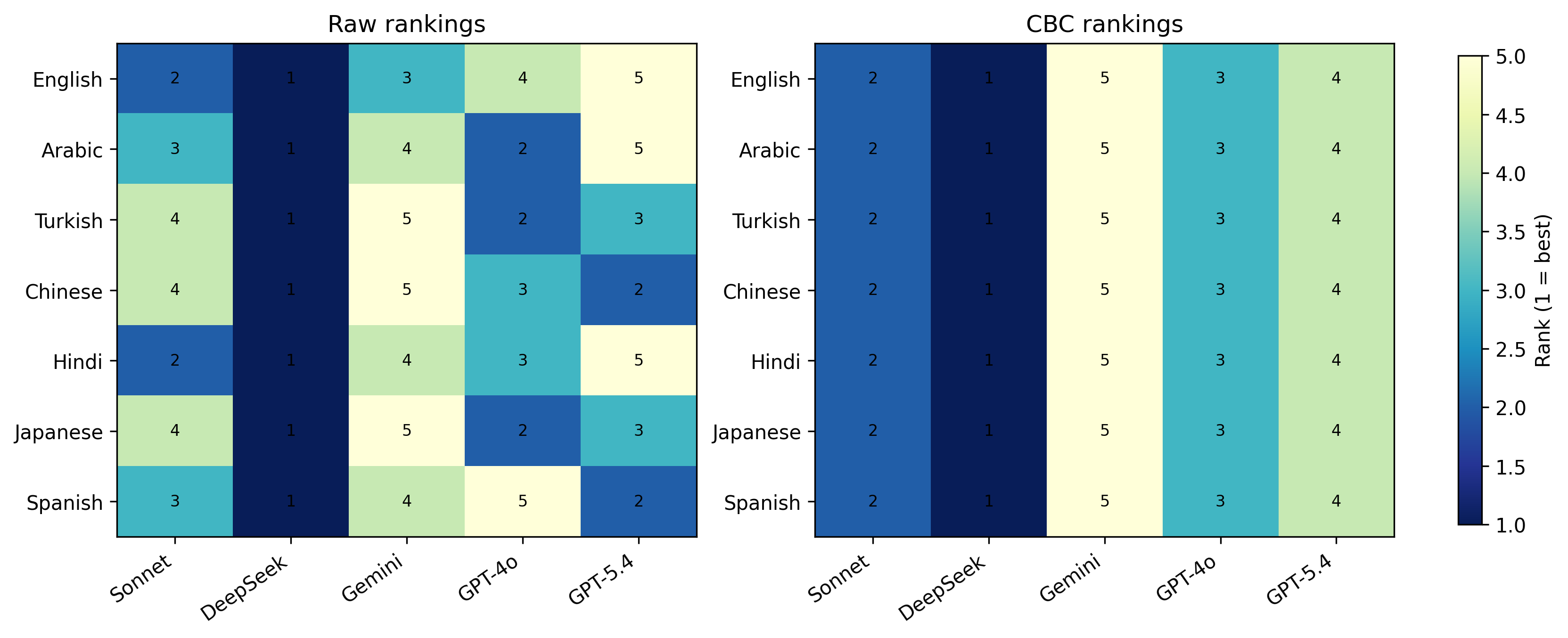}
\caption{Evaluator ranks before (left) and after (right) CBC on the self-collected 7-language, 5-evaluator score panel built from public M-RewardBench instances. Raw rankings vary across languages, while CBC aligns all seven languages to the same order: DeepSeek $>$ Sonnet $>$ GPT-4o $>$ GPT-5.4 $>$ Gemini ($\tau=1.000$). The shared post-CBC order is a full-fit in-sample sanity check, not an independent generalization estimate.}
\label{fig:external_cbc_before_after}
\end{figure}

\subsection{Ablations}
\label{sec:ablations-summary}

The $0.906$ value at $n{=}55$ comes from an independently seeded 100-replicate task-ablation loop, whereas the $0.902$ headline uses the primary 1{,}000-replicate bootstrap/OOB loop; the small difference is Monte Carlo variation, not a different benchmark configuration.

Three ablations on the same expanded benchmark, with tables, a figure, and full discussion, are reported in Appendix~\ref{app:ablations}. (i) Varying the number of backbones $m \in \{2,3,4,5,6\}$: mean $\tau$ stays near $0.90$ for all $m$, but cross-subset variability drops sharply once $m \geq 3$ (Std.\ $0.206 \to 0.105$). (ii) Varying the number of tasks $n \in \{10,20,30,40,55\}$: mean $\tau$ rises from $0.759$ at $n{=}10$ to $0.906$ at $n{=}55$, and the mean absolute estimation error $|\hat\beta-\beta_{\text{oracle}}|$ drops from $2.01$ to $0.67$, consistent with the $O(1/\sqrt{n})$ rate. For practitioners with $m \geq 3$ backbones we recommend $n \geq 30$ as a stable minimum and $n \approx 52$ as a more conservative target. (iii) Splitting by requirement type, CBC improves operational checks from $\tau = 0.770$ to $0.789$ and semantic checks from $0.737$ to $0.840$; the larger semantic gain is consistent with semantic judgments being more language-sensitive, while operational checks were already relatively stable across languages.

\section{Analysis}
\label{sec:analysis}

To corroborate the rank-based results with an information-theoretic quantity, we estimate $I(\text{Score}; \text{Language} \mid \text{Backbone}, \text{Task})$ from the framework-level scores. Each $(\text{task}, \text{backbone})$ slice contains 24 observations (8 languages $\times$ 3 frameworks); we apply a KSG-style nearest-neighbor estimator \citep{kraskov2004estimating} per slice, average over all 330 slices, and report a permutation-debiased value (500 shuffles per slice). The conditional MI is $0.178$ nats (permutation interval $[0.169, 0.187]$, one-sided $p < 0.002$): small in absolute terms but statistically reliable, and exactly the non-redundant signal that $\beta(\ell,b)$ represents. The seven verified pairwise reversals (Table~\ref{tab:diversity}) are exactly the cross-family pairs among GPT-4o, GPT-5.4, Sonnet, and DeepSeek; Qwen is Pareto-dominated in this panel and Gemini is top- or second-ranked in every language, so neither produces a sign change. A rigorous regression against standardized model-card metadata requires information that current cards do not consistently disclose and remains future work.

\section{Discussion}
\label{sec:discussion}

\paragraph{Theory meets practice.} The pairwise rank reversals tell practitioners that searching for a language-neutral backbone is unreliable; Proposition~\ref{thm:identifiability} shows they already have the data needed to estimate the interaction; and CBC tells them how to remove it when language-invariant ranking is the deployment objective. Unlike Dawid--Skene \citep{dawid1979maximum}, which models annotator-level reliability on categorical labels, CBC operates on continuous pointwise scores and exploits the two-way (language$\times$backbone) structure rather than treating each language-backbone pair as an independent annotator. This structural assumption replaces per-annotator confusion matrices with a single scalar interaction term per language-backbone pair, which is what enables stable estimation with $n=55$ tasks.

\paragraph{Cost and applicability.} CBC requires zero human annotations; by comparison, \citet{hada-etal-2024-large} use 20{,}000 native-speaker annotations across eight languages. Our compute cost is \$54.93 for the three-language extension and \$195.3 for the M-RewardBench panel (Appendix~\ref{app:costs}). The same two-way decomposition and calibrator apply beyond agentic code evaluation, to any setting where multiple judge backbones produce continuous scores on shared items across languages, including RLHF reward modeling, safety classification, automated grading, and machine-translation evaluation.

\section{Conclusion}
\label{sec:conclusion}

Multilingual LLM-judge rankings reverse across prompt languages, and the resulting language-backbone interaction is recoverable without human labels by double-centering the multi-evaluator score matrix. On the eight-language Agent-as-a-Judge benchmark, CBC raises held-out cross-task rank consistency $\tau$ from $0.650$ to $0.902$ and agrees with the held-out additive-model oracle in 100\% of per-language decisions versus 68.5\% for raw scores. On a separately collected M-RewardBench panel, $\tau$ rises from $0.430$ to $0.900$, while agreement with the public human gold preferences rises from 68.7\% to 76.6\% (+7.9 percentage points, 95\% CI $[6.0, 9.9]$). These internal gains are consistency and model-reference diagnostics rather than objective correctness; the human-anchor gain is supportive aggregate evidence and our strongest external evidence of downstream usefulness, not a complete per-language diagnosis of bias versus genuine evaluator specialization. The estimator is the textbook two-way ANOVA interaction-recovery operation under sum-to-zero contrasts; our contribution is its application to multilingual LLM-judge calibration, the explicit finite-sample concentration bound (Proposition~\ref{prop:convergence}), and the unbiasedness result under task-language misspecification (Proposition~\ref{prop:misspec}).

Two directions remain open: three-way effects and zero-shot transfer to unseen languages. We recommend at least three evaluator backbones, rank-reversal testing, $\hat\beta(\ell, b)$ reporting, and CBC before any model-selection claim.

\section*{Limitations}
\label{sec:limitations}
\textbf{Three-way misspecification.} Proposition~\ref{prop:misspec} establishes unbiasedness for two-way task-language interactions, but it does not cover genuine three-way effects $\gamma(t,\ell,b)$. If some backbones fail on particular task families only in particular languages, then the language$\times$backbone residual is no longer a pure evaluation interaction: it mixes stable language-backbone interaction with task-conditional failure modes. In that regime, double centering can partially absorb real performance heterogeneity rather than only deployment-relevant evaluation variation, and the simplified-model error bars need not carry over unchanged. This is the sharpest structural limitation of the two-way additive model and the main reason we interpret $\hat\beta(\ell,b)$ as a benchmark-level interaction estimate rather than a universal property of a model family.

\textbf{Observed-language calibration versus extrapolation.} Our strongest results are in the observed-language setting, where CBC improves mean pairwise Kendall $\tau$ from $0.650$ to $0.902$. The leave-one-language-out (LOLO) result is an exploratory extrapolation diagnostic: when one language is held out and $\hat\beta(\ell_h,b)$ must be extrapolated heuristically, the mean held-out-to-training agreement improves only from $0.649$ to $0.740$. That is still positive, but clearly smaller than the observed-language gain. We therefore do not claim that the current CBC estimator solves zero-shot transfer to unseen languages; rather, it provides strong correction when all target languages are observed, plus a modest diagnostic signal under simple LOLO heuristics.

\textbf{Complete shared-item panels.} The closed-form estimator and its finite-sample guarantees assume a complete, balanced panel in which every shared task or item is scored in every language-backbone cell. LOLO addresses an unobserved language, not arbitrary missing or unbalanced cells; with partial coverage, naive double-centering need not identify the interaction. Weighted least squares for unequal coverage, matrix completion, and hierarchical mixed-effects or shrinkage estimators are natural future directions for small or incomplete panels; we leave these extensions to future work rather than adding an ad hoc missing-cell experiment.

\textbf{Interaction versus shared language-level effects.} CBC is designed to remove the interaction term $\beta(\ell,b)$, not a language effect shared uniformly across all backbones. If all evaluators were systematically 5 points harsher in one language, that shift would vanish under double centering and remain uncorrected by CBC. Addressing that kind of shared language-level shift requires an external anchor such as human judgments or a trusted calibrated reference model. Our small human-anchor experiment is included only as a downstream validation check, not as part of the CBC estimator itself.

\textbf{Finite-sample uncertainty.} The finite-sample concentration bound is conservative in practice. In the present benchmark ($m=6$, $k=8$, $n=55$), the simultaneous high-probability radius is about $10.0$ points at $\varepsilon=0.05$ using $\hat\sigma=22.24$. The largest observed interaction cells, such as GPT-4o in Spanish ($-20.61$) or GPT-4o in English ($+11.84$) in Figure~\ref{fig:interaction_heatmap}, sit comfortably outside this radius, but moderate cells such as GPT-5.4 in Swahili ($+8.39$) or English ($-7.44$) in the same figure fall within it. The empirical signal is therefore strong for the largest language-backbone affinities, but smaller cells should not be over-interpreted as if they were known with negligible uncertainty.

\textbf{Distributional assumptions.} Proposition~\ref{prop:convergence} assumes independent Gaussian noise with constant variance. That assumption is analytically convenient, but real judge noise is likely heteroskedastic, heavy-tailed, and partly structured by prompt format, task family, or provider behavior. The consistency result only requires weaker moment conditions, but the explicit finite-sample radius does depend on the Gaussian tail calculation. More general sub-Gaussian, robust, or heteroskedastic concentration results would strengthen the theory and make the uncertainty statement less model-dependent.

\textbf{Validation data construction.} The scaled M-RewardBench result uses public benchmark \emph{instances}, but the evaluator score matrix itself was collected by us rather than released by the benchmark authors. This is the right design for CBC, since the public release does not provide evaluator-by-language scores, but it also means the validation panel inherits our rubric, prompting template, provider routing, and parser assumptions. The result therefore demonstrates generalization to an external \emph{instance source}, not validation against an externally supplied fixed score matrix.

\textbf{Benchmark scope.} The internal benchmark still contains only 55 DevAI tasks from one agentic code-evaluation family, even after expansion to eight languages. That is enough to show large and repeatable language$\times$backbone interactions, but it is not enough to claim universality across all agentic coding tasks, all judge prompts, or all evaluation formats. The external public-instance panel broadens the evidence substantially, yet it also focuses on one preference-style benchmark family. Broader validation across additional task domains, more languages, and more evaluator families would strengthen both the empirical claims and the practical scope of CBC.

\textbf{Backbone set and benchmark dependence.} Our conclusions are always relative to a finite set of judge backbones and a fixed benchmark distribution. The non-dominance observation depends on empirical rank reversal in the observed pool; adding or removing a backbone can change whether that condition is satisfied, how strong $\delta$ appears, and which ordering CBC aligns to. Likewise, the estimated $\beta(\ell,b)$ matrix is benchmark-dependent: it is shaped by the prompts, tasks, and scoring rubric used in this study rather than representing a context-free property of the models.

\textbf{Open theoretical directions.} The current non-dominance result is deterministic and conditional on observed reversal. We do not yet provide a fully probabilistic impossibility theorem that would quantify how likely universal dominance is under a generative prior over language-backbone affinities. Developing such a result, together with sharper uncertainty characterizations and explicit models of unseen-language transfer, is an important next step if this line of work is to mature from a benchmark-specific calibration method into a broader statistical theory of multilingual evaluation.

\section*{Reproducibility Statement}

All proofs are provided in full (main text and appendix). Experimental validation builds on the previously introduced five-language multilingual Agent-as-a-Judge benchmark \citep{mahmood2026multilingualpromptlocalizationagentasajudge}, adds our three-language extension and the rank-reversal/CBC analyses, and uses a separately collected evaluator panel over the public M-RewardBench dataset of \citet{gureja-etal-2025-rewardbench}. Because provider billing exports were not preserved consistently, Appendix~\ref{app:costs} reports estimated list-price API costs computed from recorded token counts in the saved artifacts and public model rates. The public repository at \url{https://github.com/KurbanIntelligenceLab/multilingual-judge-calibration} contains the CBC implementation, analysis scripts, external collection pipeline, requirements file with pinned versions, and the derived evaluator-score matrices used in the reported analyses.

\section*{Ethics Statement}
\label{sec:ethics}

This work uses no human annotations and no personal data. The multilingual Agent-as-a-Judge benchmark consists of synthetic software-engineering tasks; the M-RewardBench instances and their gold preference labels are public and released under their original license \citep{gureja-etal-2025-rewardbench}. All evaluator runs were obtained through official commercial APIs under standard terms of service. No human subjects were involved at any stage.

Two scope concerns are worth surfacing. First, CBC calibrates the language-backbone interaction $\beta(\ell, b)$ only for languages observed in the calibration panel; languages outside the panel, including most low-resource languages, do not receive the benefit of calibration and may continue to receive systematically shifted evaluations. Second, CBC corrects the interaction term but not a shared language-level shift $g(\ell)$ (Section~\ref{sec:framework}); if every evaluator backbone systematically under- or over-scores a particular language by the same amount, that shift remains. Treating CBC-aligned rankings as ground truth without an external anchor therefore risks entrenching a shared shift rather than detecting it. Practitioners deploying CBC for fairness-sensitive decisions should combine it with a human-anchor check, such as the gold-preference comparison in Section~\ref{sec:experiments}.


\appendix

\section{Proof Details and Identifiability Scope}
\label{app:proofs}

\paragraph{Witness language pairs (Table~\ref{tab:diversity}).} For each of the seven significant pairs in Table~\ref{tab:diversity}, the witness language pair $(\ell_a, \ell_b)$ used in the test is the language pair attaining the most negative product $d_{ij}(\ell_a)\cdot d_{ij}(\ell_b)$ over the eight observed languages (En, Ar, Zh, Hi, Ja, Es, Tr, Sw): GPT-4o vs.\ GPT-5.4: En/Es; GPT-4o vs.\ Sonnet: Hi/Es; GPT-4o vs.\ Gemini: En/Es; GPT-4o vs.\ DeepSeek: En/Es; GPT-5.4 vs.\ Sonnet: Ja/Sw; GPT-5.4 vs.\ DeepSeek: Zh/Sw; Sonnet vs.\ DeepSeek: Ar/Hi. The remaining eight pairs achieved $\delta = 0$ (no negative product over any observed language pair) and are reported with adjusted $p = 1.000$.

\subsection{Identifiability Proof in Full}

We restate the simplified model:
\[
S(t,\ell,b) = \mu(t) + \alpha(b) + \beta(\ell,b) + \epsilon(t,\ell,b),
\]
with a complete balanced panel over tasks, languages, and backbones. Define the population cell mean
\[
M(\ell,b) \triangleq \frac{1}{n}\sum_{t \in \mathcal{T}} \mathbb{E}[S(t,\ell,b)]
= \bar{\mu} + \alpha(b) + \beta(\ell,b),
\]
where $\bar{\mu} \triangleq \frac{1}{n}\sum_t \mu(t)$ and $\mathbb{E}[\epsilon(t,\ell,b)] = 0$.

Let
\[
M(\cdot,b) \triangleq \tfrac{1}{k}\sum_{\ell'} M(\ell',b),
\]
\[
M(\ell,\cdot) \triangleq \tfrac{1}{m}\sum_{b'} M(\ell,b'),
\]
\[
M(\cdot,\cdot) \triangleq \tfrac{1}{mk}\sum_{\ell',b'} M(\ell',b').
\]
Under the normalization constraints
\[
\sum_{\ell} \beta(\ell,b) = 0 \;\; \forall b,
\quad
\sum_b \beta(\ell,b) = 0 \;\; \forall \ell,
\]
we have
\[
M(\cdot,b) = \bar{\mu} + \alpha(b),
\]
\[
M(\ell,\cdot) = \bar{\mu} + \bar{\alpha},
\quad
M(\cdot,\cdot) = \bar{\mu} + \bar{\alpha},
\]
where $\bar{\alpha} \triangleq \frac{1}{m}\sum_b \alpha(b)$. Therefore
\begin{align*}
&M(\ell,b) - M(\cdot,b) - M(\ell,\cdot) + M(\cdot,\cdot) \\
&\quad= \big(\bar{\mu} + \alpha(b) + \beta(\ell,b)\big) - \big(\bar{\mu} + \alpha(b)\big) \\
&\qquad - \big(\bar{\mu} + \bar{\alpha}\big) + \big(\bar{\mu} + \bar{\alpha}\big) \\
&\quad= \beta(\ell,b).
\end{align*}
Hence the centered interaction matrix is recovered exactly by double centering. Uniqueness follows immediately: if another matrix $\beta'$ with the same zero-sum constraints produced the same cell means, double centering those means would yield both $\beta$ and $\beta'$, so $\beta'=\beta$.

\subsection{What Label-Free Identifiability Means}

The phrase ``identifiable without labels'' should be interpreted narrowly. The observable data determine the centered interaction matrix $\beta(\ell,b)$ for \emph{observed} language-backbone cells. This is sufficient for calibration on those observed languages. However:
\begin{itemize}
\item No human gold labels are needed to estimate $\beta(\ell,b)$.
\item The absolute location of $\mu(t)$ and $\alpha(b)$ is not unique; additive shifts can be absorbed between them.
\item The interaction for an unseen language $\ell'$ or unseen backbone $b'$ is not identifiable without collecting scores for that new cell.
\item A three-way effect $\gamma(t,\ell,b)$ is not identified by the two-way model and cannot be separated from $\beta$ without additional structure.
\end{itemize}

\subsection{Convergence Bound with Exact Constants}

Write the cell-average noise as
\[
\bar{\epsilon}(\ell,b) \triangleq \frac{1}{n}\sum_{t \in \mathcal{T}} \epsilon(t,\ell,b).
\]
Under the Gaussian assumption in Proposition~\ref{prop:convergence}, each $\bar{\epsilon}(\ell,b)$ is Gaussian with mean $0$ and variance $\sigma^2/n$. Since $\hat{\beta}$ is the double-centered cell mean,
\[
\hat{\beta}(\ell,b) - \beta(\ell,b)
= \bar{\epsilon}(\ell,b) - \bar{\epsilon}(\cdot,b) - \bar{\epsilon}(\ell,\cdot) + \bar{\epsilon}(\cdot,\cdot).
\]

For a fixed target cell $(\ell,b)$, the coefficient of each averaged noise term is:
\[
c_{\ell'b'} =
\begin{cases}
1 - \frac{1}{k} - \frac{1}{m} + \frac{1}{mk}, & \ell'=\ell,\; b'=b,\\[4pt]
-\frac{1}{k} + \frac{1}{mk}, & \ell' \neq \ell,\; b'=b,\\[4pt]
-\frac{1}{m} + \frac{1}{mk}, & \ell'=\ell,\; b' \neq b,\\[4pt]
\frac{1}{mk}, & \ell' \neq \ell,\; b' \neq b.
\end{cases}
\]
Because the cell-average noises are independent across $(\ell',b')$,
\[
\mathrm{Var}\!\left[\hat{\beta}(\ell,b)-\beta(\ell,b)\right]
= \frac{\sigma^2}{n}\sum_{\ell',b'} c_{\ell'b'}^2.
\]
Grouping the four coefficient types gives
{\small
\begin{align*}
\sum_{\ell',b'} c_{\ell'b'}^2
&=\Big(1-\tfrac{1}{k}-\tfrac{1}{m}+\tfrac{1}{mk}\Big)^2 \\
&\quad + (k-1)\Big(\tfrac{1}{k}-\tfrac{1}{mk}\Big)^2 \\
&\quad + (m-1)\Big(\tfrac{1}{m}-\tfrac{1}{mk}\Big)^2 \\
&\quad + (m-1)(k-1)\Big(\tfrac{1}{mk}\Big)^2 \\
&= \Big(1-\tfrac{1}{m}\Big)\Big(1-\tfrac{1}{k}\Big).
\end{align*}
}
Therefore
\[
\hat{\beta}(\ell,b)-\beta(\ell,b)
\sim
\mathcal{N}\!\left(
0,\;
\tfrac{\sigma^2}{n}\big(1-\tfrac{1}{m}\big)\big(1-\tfrac{1}{k}\big)
\right).
\]
Applying the standard Gaussian tail bound yields, for any $x>0$,
{\footnotesize
\begin{align*}
&\Pr\!\left(|\hat{\beta}(\ell,b)-\beta(\ell,b)| \ge x\right) \\
&\qquad \le 2\exp\!\left(-\tfrac{n x^2}{2\sigma^2(1-1/m)(1-1/k)}\right).
\end{align*}
}
Finally, applying a union bound over the $mk$ observed language-backbone cells gives
{\footnotesize
\begin{align*}
&\Pr\!\Big(\max_{\ell,b} |\hat{\beta}(\ell,b)-\beta(\ell,b)| \ge x\Big) \\
&\qquad \le 2mk \exp\!\Big(-\tfrac{n x^2}{2\sigma^2(1-1/m)(1-1/k)}\Big),
\end{align*}
}
and solving for $x$ proves Proposition~\ref{prop:convergence}.

\subsection{Consistency Proof in Full}

For each fixed $(\ell,b)$, the cell mean
\[
\bar{S}(\ell,b) = \frac{1}{n}\sum_{t \in \mathcal{T}} S(t,\ell,b)
\]
converges in probability to $M(\ell,b)$ by the law of large numbers, since the task-level observations are independent and have finite first moment. The CBC estimator is a linear transformation of the finite collection of cell means:
\[
\hat{\beta}(\ell,b)
= \bar{S}(\ell,b) - \bar{S}(\cdot,b) - \bar{S}(\ell,\cdot) + \bar{S}(\cdot,\cdot).
\]
Linear combinations preserve convergence in probability, so $\hat{\beta}(\ell,b)$ converges in probability to
\[
M(\ell,b) - M(\cdot,b) - M(\ell,\cdot) + M(\cdot,\cdot)
= \beta(\ell,b),
\]
where the final equality follows from the identifiability argument above. This proves Proposition~\ref{prop:consistency}.

\section{Additional Ablations}
\label{app:ablations}

This appendix contains the three ablation studies referenced from Section~\ref{sec:ablations-summary}, with full tables, a convergence figure, and the original discussion. All ablations use the same expanded eight-language Agent-as-a-Judge benchmark. The backbone and requirement-type ablations use the observed-language bootstrap/OOB protocol, while the task-count ablation uses an independently seeded 100-replicate task-ablation loop described below.

\subsection{Ablation: Number of Backbones ($m$)}

Proposition~\ref{thm:identifiability} requires at least two backbones and two languages; more backbones reduce estimation variance. We test sensitivity to $m$ by subsampling backbone subsets.

\begin{table}[h]
\centering\footnotesize
\renewcommand{\arraystretch}{1.25}
\caption{CBC calibration quality as a function of the number of backbones $m$. Rank $\tau$ is averaged over all $\binom{6}{m}$ subsets using the observed-language bootstrap/OOB protocol.}
\label{tab:ablation_m}
\begin{tabular}{@{}l rrrrr@{}}
\toprule
$m$ & 2 & 3 & 4 & 5 & 6 \\
\midrule
Rank $\tau$ & 0.901 & 0.907 & 0.900 & 0.898 & 0.911 \\
Std.        & 0.206 & 0.105 & 0.075 & 0.044 & --- \\
\bottomrule
\end{tabular}
\end{table}

CBC remains strong even with small backbone subsets (Table~\ref{tab:ablation_m}): the mean $\tau$ is already 0.901 with $m{=}2$ and stays near 0.90 for $m \geq 3$. However, the $m{=}2$ setting is much less stable than the larger subsets (Std.\ $= 0.206$), so the main effect of larger $m$ is reduced variability across subsets rather than a large increase in mean performance.

\subsection{Ablation: Number of Tasks ($n$)}

The convergence bound (Proposition~\ref{prop:convergence}) scales as $O(1/\sqrt{n})$. We test empirically by subsampling tasks. The grid $n \in \{10,20,30,40,55\}$ was chosen to show a simple progression from very small task panels to the full benchmark, using 10-task increments plus the full-data endpoint.

\begin{table}[h]
\centering\footnotesize
\renewcommand{\arraystretch}{1.25}
\caption{CBC calibration quality as a function of the number of tasks $n$. Averaged over 100 random subsamples.}
\label{tab:ablation_n}
\begin{tabular}{@{}l rrrrr@{}}
\toprule
$n$ & 10 & 20 & 30 & 40 & 55 \\
\midrule
Rank $\tau$ & 0.759 & 0.838 & 0.853 & 0.891 & 0.906 \\
Std.        & 0.103 & 0.061 & 0.057 & 0.054 & 0.051 \\
\bottomrule
\end{tabular}
\end{table}

The task ablation shows that CBC remains useful even with only 10 tasks, but improves and becomes more stable as more tasks are available (Table~\ref{tab:ablation_n}). The mean absolute estimation error $|\hat\beta-\beta_{\text{oracle}}|$ drops from 2.01 at $n{=}10$ to 0.67 at $n{=}55$, consistent with the convergence story (Figure~\ref{fig:convergence}).

For a simple practitioner-facing power check, consider the largest observed interaction magnitude, $|\hat\beta|_{\max} = 20.61$ (GPT-4o in Spanish). Under Proposition~\ref{prop:convergence}, the simultaneous high-probability radius drops below half of that scale once $n > 51$, i.e., at about 52 tasks in this benchmark. Empirically, CBC is already fairly stable by $n{=}30$ (mean pairwise $\tau = 0.853$ and mean absolute interaction-estimation error $1.11$), with smaller gains thereafter. For practitioners deploying CBC with $m \geq 3$ backbones, we therefore recommend $n \geq 30$ tasks as a practical minimum for stable interaction estimates, while $n \approx 52$ is a more conservative target if one wants the Proposition~\ref{prop:convergence} radius to fall below half of the largest observed $|\beta|$.

\begin{figure}[h]
\centering
\includegraphics[width=0.98\columnwidth]{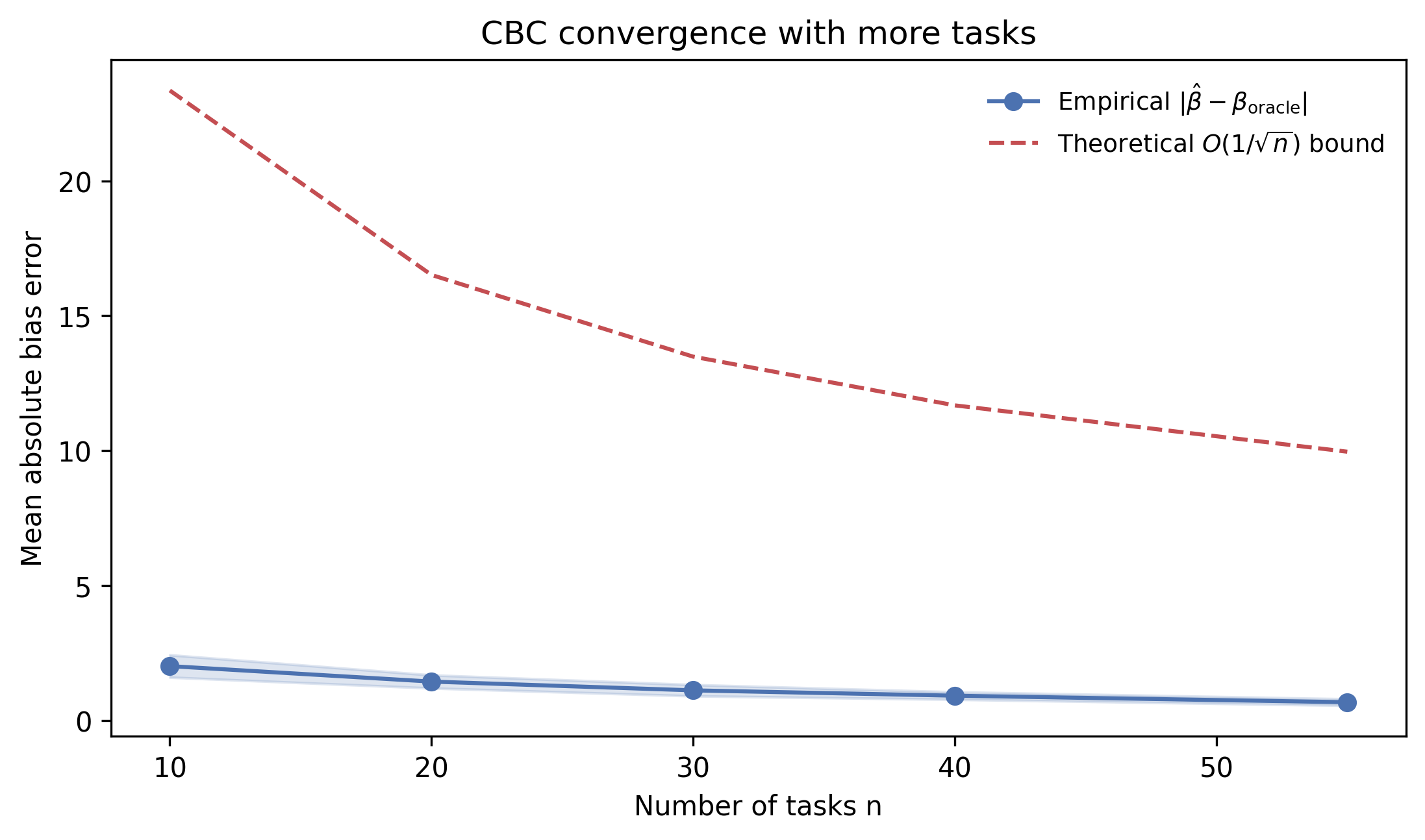}
\caption{Convergence of CBC interaction-estimation error as a function of the number of tasks $n$. The empirical mean absolute error $|\hat\beta-\beta_{\mathrm{oracle}}|$ decreases steadily with more tasks. The dashed reference curve plots the explicit Proposition~\ref{prop:convergence} constant using $\hat{\sigma}=22.24$, $\varepsilon=0.05$, $m=6$, and $k=8$, so it should be read as the paper's conservative finite-sample bound rather than as a generic asymptotic envelope.}
\label{fig:convergence}
\end{figure}

\subsection{Ablation: Requirement-Type Decomposition}

We fit CBC separately for operational types (Data Loading, Training) and semantic types (Model Construction, Evaluation Metrics). The pattern is positive in both cases, though stronger for semantic requirements: for operational requirements, CBC improves stability from $\tau = 0.770$ to $\tau = 0.789$, while for semantic requirements it improves stability from $\tau = 0.737$ to $\tau = 0.840$. This suggests that CBC's largest benefit comes from correcting cross-language instability in semantically richer judgments, but the expanded benchmark also reveals a smaller gain on operational checks.

One plausible explanation is that the operational slice is thinner at the requirement level. Under the benchmark's requirement taxonomy, operational requirements comprise 82 items in total (62 Data Loading + 20 Training), whereas the semantic slice contains 114 items (64 Model Construction + 50 Evaluation Metrics). Thus the operational split has less per-task signal available for estimating the language-backbone interaction. At the same time, the gap is not purely a task-count artifact: operational requirements still appear in 54 of the 55 tasks, so the weaker gain is better understood as a combination of lower slice size and lower raw instability. Operational checks are often closer to concrete existence or execution conditions, which are already relatively stable across languages ($\tau = 0.770$ before calibration), whereas semantic checks require more language-sensitive judgment about modeling choices and evaluation adequacy, leaving more multilingual interaction for CBC to remove.

\section{External Validation Details}
\label{app:external_details}

This appendix expands the M-RewardBench validation setting summarized in Section~\ref{sec:expanded}.

\paragraph{Collection pipeline.} M-RewardBench provides aligned multilingual preference instances across 23 languages, but does not release evaluator-by-language score matrices, only the public instances. We built a collection pipeline that queries each evaluator with a fixed 1--5 pointwise rubric on chosen/rejected responses, then converts chosen-minus-rejected margins into CBC-ready task$\times$language$\times$evaluator tables. The completed panel covers 7 overlapping languages (English, Arabic, Turkish, Simplified Chinese, Hindi, Japanese, Spanish), 1{,}500 aligned items per language (10{,}500 language-item instances total), and 5 evaluator backbones (Claude Sonnet 4.6, DeepSeek-V3.2, Gemini 3 Flash Preview, GPT-4o, GPT-5.4). This corresponds to 52{,}500 judged preference pairs, or 105{,}000 pointwise evaluator calls once chosen and rejected responses are scored separately.

\paragraph{Adapted judge-aware Bradley--Terry--Luce baseline.} As the substantive external pairwise comparator, we adapt Xu et al.'s \citeyearpar{xu2026judgeaware} judge-aware Bradley--Terry--Luce model to the evaluator-ranking target. For each language we convert item-level chosen-minus-rejected margins into pairwise evaluator wins (sign of the margin difference), then fit the BTL likelihood with judge-specific discrimination parameters. Under the same bootstrap/OOB protocol as the main benchmark, this baseline reaches $\tau = 0.405$ (95\% CI $[0.352, 0.467]$), close to Raw and well below CBC. The comparison is informative because the two methods operate in different input regimes: the Xu-style model is appropriate when only pairwise wins/ties are available and confidence intervals over latent rankings are desired, whereas CBC exploits the stronger pointwise-margin regime and outperforms this comparator when the dominant issue is an additive language-backbone interaction in observed score matrices.

\begin{table*}[h]
\centering\footnotesize
\renewcommand{\arraystretch}{1.2}
\setlength{\tabcolsep}{6pt}
\caption{Estimated API cost breakdown for the two added experimental components. ``Calls'' counts direct model invocations; the M-RewardBench panel contains 1{,}500 items per language (10{,}500 language-item instances) and scores both chosen and rejected responses separately, so 52{,}500 judged pairs correspond to 105{,}000 pointwise calls. Costs are estimates from recorded token counts and public list prices, not billing exports.}
\label{tab:cost_breakdown}
\begin{tabular}{@{}lrrrr@{}}
\toprule
\textbf{Component} & \textbf{Calls} & \textbf{Input tokens} & \textbf{Output tokens} & \textbf{Est.\ cost} \\
\midrule
3-language extension (JA / ES / SW) & 19{,}710 & 25.68M & 6.64M & \$54.93 \\
M-RewardBench 1{,}500-items-per-language panel & 105{,}000 & 52.06M & 9.85M & \$195.30 \\
\bottomrule
\end{tabular}
\end{table*}

\paragraph{In-sample full-fit diagnostic.} The bootstrap train/OOB number is the headline in Table~\ref{tab:external_validation}. We additionally report an in-sample full-fit diagnostic: $\hat\beta$ is estimated on all 1{,}500 items per language and $\tau$ is computed on that same fitted panel. The resulting $\tau = 1.000$ should be read as a sanity check on the recovered shared ordering, not as an independent generalization estimate. The adapted Xu baseline stays at the raw full-panel level ($\tau = 0.410$), consistent with the fact that it discards margin magnitudes and retains only pairwise orderings within each item.

\paragraph{Panel-construction notes.} We attempted to include a Qwen3 evaluator in the M-RewardBench panel but excluded it from the completed panel because its provider run produced incomplete chosen/rejected outputs and never reached a clean terminal state. M-Prometheus \citep{pombal2025mprometheus} is positioned by its authors as a multilingual judge-training resource rather than a public evaluator-output release on a shared benchmark, so we did not adopt it as a validation panel here.

\section{Estimated API Cost Breakdown}
\label{app:costs}

Table~\ref{tab:cost_breakdown} reports estimated API costs for the two added experimental components that required fresh model calls. These are \emph{not} billing-export totals. Instead, we aggregate the recorded input/output token counts stored in the saved artifacts and multiply by contemporaneous public list prices for the corresponding model versions. This was necessary because provider-side dollar fields were preserved inconsistently across runs.

\subsection{Appendix Note on Quantile Normalization}
\label{app:quantile_note}

We audited the quantile-normalization baseline and replaced an earlier column-wise empirical-CDF transform with the standard target-distribution formulation used by \texttt{preprocessCore::normalize.quantiles}. The exact implementation used for Table~\ref{tab:calibration} is:

\begin{small}
\begin{verbatim}
def col_quantile(matrix, train_matrix):
    sorted_train = np.sort(
        train_matrix.to_numpy(dtype=float),
        axis=0)
    target = sorted_train.mean(axis=1)
    qdf = pd.DataFrame(
        index=matrix.index,
        columns=matrix.columns, dtype=float)
    for b in matrix.columns:
        v = matrix[b].to_numpy(dtype=float)
        order = np.argsort(v, kind="mergesort")
        sv = v[order]
        ns = target.copy()
        s = 0
        while s < len(sv):
            e = s + 1
            while e < len(sv) and sv[e] == sv[s]:
                e += 1
            if e - s > 1:
                ns[s:e] = float(np.mean(ns[s:e]))
            s = e
        nv = np.empty_like(v, dtype=float)
        nv[order] = ns
        qdf[b] = nv
    return qdf
\end{verbatim}
\end{small}

On the full benchmark score matrix, this implementation matches an independent NumPy implementation of the standard sort-average target-distribution algorithm exactly (maximum absolute difference $0.0$).

\subsection{Appendix Note on Weighted CBC}

We also explored a weighted variant of CBC that replaces the uniform backbone average with weights $w_{b'} \propto \mathrm{Var}_\ell[\bar{S}(\ell, b')]^{-1}$. On the expanded multilingual Agent-as-a-Judge benchmark, this variant matched uniform CBC on the observed-language bootstrap/OOB metric ($\tau = 0.902$ for both), so we omit it from the main comparison table and do not treat it as a separate method in the main paper.

\subsection{Appendix Note on Dawid--Skene EM}
\label{app:ds_em_note}

We also tested a Dawid--Skene EM baseline \citep{dawid1979maximum} by binarizing requirement outcomes and treating each language-backbone pair as an annotator. On this benchmark it produced a highly unstable bootstrap summary ($\tau = 0.446$ with 95\% CI $[-0.042, 1.000]$), indicating that the binary latent-class model was a poor fit for our continuous score-matrix setting. Because this interval spans nearly the full range of possible outcomes, we do not treat Dawid--Skene EM as a useful main-table baseline here.

\end{document}